\documentclass[twoside,11pt]{article}

\usepackage{amsmath,amsthm}
\usepackage{tikz-cd}
\usepackage{subcaption}

\theoremstyle{definition}
\newtheorem{thm}{Theorem}
\newtheorem{exmp}[thm]{Example}

\usepackage[preprint]{jmlr2e}

\usepackage{lastpage}
\jmlrheading{}{2026}{}{}{}{}{Tim Mangliers, Bernhard Mössner and Benjamin Himpel}

\ShortHeadings{Spectral Convolution on Orbifolds}{Mangliers, Mössner and Himpel}
\firstpageno{1}

\begin{document}
\title{Spectral Convolution on Orbifolds for Geometric Deep Learning}

\author{\name Tim Mangliers \email Tim.Mangliers@Reutlingen-University.de \\
	\name Bernhard Mössner \email Bernhard.Moessner@Reutlingen-University.de \\
	\name Benjamin Himpel \email Benjamin.Himpel@Reutlingen-University.de \\
	\addr Department of Computer Science\\
	Reutlingen University, Germany
}

\editor{My editor}
\maketitle

\begin{abstract}%
Geometric deep learning (GDL) deals with supervised learning on data domains that go beyond Euclidean structure, such as data with graph or manifold structure. Due to the demand that arises from application-related data, there is a need to identify further topological and geometric structures with which these use cases can be made accessible to machine learning. There are various techniques, such as spectral convolution, that form the basic building blocks for some CNN-like architectures on non-Euclidean data. In this paper, the concept of spectral convolution on orbifolds is introduced. This provides a building block for making learning on orbifold-structured data accessible using GDL. The theory discussed is illustrated using an example from music theory.
\end{abstract}


\section{Introduction}
Supervised learning can take place on Euclidean data. However, real-world data can also have a different structure, such as being structured as graphs or manifolds \citep{Zafeiriou2022,Chowdhury2023}. The question is how architectures that achieve good results on Euclidean data, such as convolutional neural networks (CNNs), can be adapted and generalized so that effective and efficient learning can also take place on more general data domains. The development of techniques for learning on non-Euclidean data domains is associated with the field of geometric deep learning (GDL). Several different basic techniques already exist that enable learning on graphs and manifolds using generalized CNNs \citep{Henaff2015,Masci2015}. 

Fundamental to the CNN architecture is the notion of convolution. Various types of convolution can be defined on manifolds, such as spatial and spectral convolution, each exhibiting distinct advantages and disadvantages \citep{Bruna2014,Monti2017,gerken2023geometric}. Orbifolds are a mild generalization of manifolds and offer the possibility of representing a framework with which certain symmetry-structured data can be classified. It is shown in \citep{Jain2012} that an approach using stochastic generalized gradient learning on orbifolds can be used to learn on particular combinatorial structures such as graphs. However, in this paper, it is shown that a spectral convolution as it can be defined on manifolds can be generalized to orbifolds. In this way, learning on orbifolds is conceptually integrated into GDL and can benefit from its descriptive power.

The structure of the paper is as follows. First, we introduce the setting and motivation of GDL. We then develop the main theoretical contribution of this work by establishing a spectral notion of convolution on orbifolds. Subsequently, we illustrate the proposed construction on a central example drawn from musicology. Finally, we conclude with a discussion and an outlook on further domains in which the orbifold-based approach may prove fruitful.

\section{Geometric Deep Learning}
The term GDL was introduced in the project Deep LEarning on MANifolds and graphs and the paper \citep{bronstein2017geometric} as an umbrella term for deep learning techniques that go beyond Euclidean data processing and learn more general geometrically structured data. The main focus here was initially on techniques that process data in the form of manifolds and graphs. Such data can occur in the form of 3D shapes, point clouds, meshes and graphs in applications like computer graphics, particle physics, drug design based on molecule data and algorithms in social networks \citep{cao2020,sommer2020,zhou2020,cao2022,powers2023}.

The goal of GDL was then defined mainly more broadly in \citep{bronstein2021geometric}. GDL since then has aimed to systematize the field of deep learning. This is because there are a large number of deep learning architectures that are fundamentally different, at least on the surface. However, architectures can be categorized using a geometric approach. The geometric view is such that architectures are classified based on the symmetries of the data on which they operate. Symmetries mean that data and structures are preserved under certain transformations and do not change, i.e. remain invariant. These transformations can be described with the help of group theory. For example, unstructured sets and graphs have permutation invariance and grids have translation invariance.

The architectures acting on the data then respect the structure belonging to the data by structurally incorporating these symmetries themselves. This view is presented in detail in \citep{bronstein2021geometric} with the so-called GDL blueprint. In the GDL blueprint, abstract layers, such as equivariant and invariant layers, are defined with respect to group operations. This blueprint can then be used to classify existing architectures and to design new architectures based on it. For CNNs acting on Euclidean data or grid-like data, this means that the convolutional layers are shift-equivariant, i.e. a shift in the input of a layer leads to a shift in the output and shift invariance for max-pooling layers means that a shift in the input leads to the same result in the output as if no shift had been performed.
One way to generalize learning on Euclidean data to learning on non-Euclidean data is to generalize convolution to design architectures that are structurally similar to CNNs. One way to extend the classical convolution to graphs and manifolds is by spectral convolution. Spectral convolution basically works in such a way that functions on certain structures (such as graphs or manifolds) are represented as a linear combination with the help of an orthogonal basis and in this representation two different functions can simply be combined and then transformed back. The basic idea here comes from the classical convolution. Scalar functions can be transformed into a specific linear combinatorial representation using the Fourier transform; it is also said that they are transformed from the spatial domain into the frequency domain. Here, the basis of the linear combination is exactly the eigenfunctions of the Laplacian. The convolution theorem now states that a convolution of two scalar functions after their Fourier transformation in the frequency domain represents a multiplication. The graph Laplacian can be defined on graphs and a generalization of the Laplacian, the so-called Laplace-Beltrami operator, can be defined on manifolds. The projection of scalar functions onto the eigenfunctions of these operators, their multiplication, and their reverse transformation, is then the definition of spectral convolution.
The following illustrates this concept on manifolds according to \citep{Vallet2008} and \citep[p. ~54]{bronstein2021geometric}. If $\Delta$ is the Laplacian on a compact Riemannian manifold, then its eigenfunctions $\varphi_{k}$ can be used to represent square-integrable functions $f$ as Fourier series $f(x) = \sum_{k = 0}^{\infty} \left< f , \varphi_{k} \right> \varphi_{k}(x)$ using the generalized Fourier transform with Fourier coefficients $\left< f , \varphi_{k} \right> = \hat{f}_{k}$. The index is discrete, as the manifold is assumed to be compact. The spectral convolution of two functions then looks like
\[
(f * g)(x) = \sum_{k = 0}^{\infty} (\hat{f}_{k} \cdot \hat{g}_{k}) {\varphi_{k}}(x).
\]
In the next section we show that this kind of convolution can be generalized to orbifolds.

\section{Spectral Convolution on Orbifolds}
We show how spectral convolution can be defined, providing a fundamental building block for CNNs and GDL on orbifolds.

There are different ways of describing orbifolds, but for our purposes of connecting orbifolds with GDL we choose the classical differential geometric point of view.

Manifolds are objects that are locally homeomorphic to open subsets of $\mathbb{R}^n$, whereas orbifolds are locally homeomorphic to quotients of open subsets $\mathbb{R}^n$ by the action of a finite group. Manifolds can carry additional structure, so there are also manifolds whose transition maps between charts are smooth. In this case, manifolds can be equipped with a Riemannian metric. A Riemannian metric is a smooth function that assigns to each point on the manifold a scalar product in the tangent space adjacent to this point. In the same way, orbifolds can have diffeomorphisms and can then also be equipped with Riemannian metrics. The Riemannian metric can then be used to define a generalization of the Laplacian, and this differential operator can, in turn, be used to define the aforementioned spectral convolution of functions in $L^{2}(X)$, whereby this is the Hilbert space of certain measurable, complex-valued functions on the orbifold and is constructed in a manner analogous to classical measure and integration theory on manifolds. \citep{Stanhope2011,Gordon2012} 

Orbifolds can therefore be introduced like manifolds via maps of $\mathbb{R}^{n}$ with certain properties or equivalently as global quotients of manifolds.  We describe the latter approach in the succeeding paragraphs. A more detailed discussion of orbifolds can be found for instance in \citep{Adem2007}.

Throughout we work with orbifolds presented as global quotients 
\[ X = M/G, \]
where $M$ is a compact smooth Riemannian manifold and $G$ is a compact Lie group acting smoothly, effectively and almost freely on $M$ (that is, all isotropy groups are finite).
This definition is equivalent to the usual definition of an effective smooth orbifold based on local charts, mentioned at the beginning of this section, see \citep[pp.~12--13]{Adem2007} and \citep[p.~50]{Stanhope2011}.

\begin{exmp}
\label{example:2}
The discrete, non-compact Lie group $\mathbb{Z}^{2}$ acts on $\mathbb{R}^{2}$ by integer translations $\tau (x,y) := (x + m, y + n)$. This action is smooth, proper and free; therefore, the resulting quotient is a  compact smooth manifold, commonly known as the (topological) torus $\mathbb{T}_{1}^{2}$.
Reflecting points across the diagonal in $\mathbb{R}^2$ factors through $\mathbb{T}_1^2$ as a smooth action of the finite Lie group $S_2$, which is proper and almost free. The resulting quotient $\mathbb{T}_1^2 / S_2 =: \mathcal{C}^{2}_{1}$ is therefore an effective smooth orbifold.
\end{exmp}

A $G$-invariant metric $g_M$ on $M$ can be constructed by averaging. 
This metric then descends to a well-defined Riemannian metric $g_X$ on the quotient $X = M / G$, endowing $X$ with the structure of a Riemannian orbifold.
It is then possible to define smooth functions on $X$, and the pullback $\pi^*$ associated with the projection $\pi : M \to X$, defined by $(\pi^* f)(x) = f(\pi(x))$, 
induces an isomorphism between smooth functions on $X$ and smooth $G$-invariant functions on $M$ \citep[p.~50]{Stanhope2011}.

\begin{exmp}
	Consider the unit flat torus $\mathbb{T}^2_1$ with the Riemannian metric induced from the standard Euclidean metric on $\mathbb{R}^2$. 
	Let $S_2$ act as in Example~\ref{example:2}, namely by swapping the coordinates. 
	Since the Euclidean metric is invariant under isometries, in particular under this reflection, the induced metric on $\mathbb{T}^2_1$ is $S_2$-invariant. 
	Hence, the quotient $\mathcal{C}^2_1$ inherits a well-defined Riemannian metric and is a compact Riemannian orbifold.
\end{exmp}

Consider the as usual defined Hilbert space
\[
L^{2}(X) = \{ f : X \to \mathbb{C} \ \text{measurable} \mid \int_X |f|^{2}\, dv_{X} < \infty \},
\]
where $dv_{X}$ denotes the orbifold volume measure induced by the Riemannian metric. 
The Laplacian $\Delta_{M} = -\operatorname{div}(\nabla_{M})$ on $M$ induces the Laplacian
on $X$ via
$
\Delta_X(f) := \Delta_M(\pi^* f).
$ Since the Laplacian $\Delta_M$ commutes with isometries \citep[pp.~99--100]{Puta2001}, it preserves $G$-invariance; hence, for any $f \in C^\infty(X)$, the function $\Delta_M(\pi^* f)$ is $G$-invariant, and the induced operator $\Delta_X$ on $X$ is well defined. Beyond that is the Laplace spectrum of $X$ contained in that of $M$,
$
\mathrm{Spec}(\Delta_{X}) \subseteq \mathrm{Spec}(\Delta_{M}).
$
See \citep[p.~50]{Stanhope2011} and \citep[p.~216]{Farsi2001} for details. Given a Hilbert space $H$ with an orthonormal basis $\{e_i\}_{i \in I} \subseteq H$, vectors $v\in H$ can be represented by the expansion 
$
v = \sum_{i \in I} \hat{v}(i)\, e_i,
$
where the coefficients are given by $\hat{v}(i) := \langle v, e_i \rangle$. 
In particular, this holds for the Hilbert space of square-summable sequences $\ell^2(\mathbb{N})$, which will be used in the following.

\begin{thm}
	\label{theo:theo_1}
	Let $X$ be an orbifold. 
	Then there exists a unitary operator
	\[
	\mathcal{F}: L^2(X) \to \ell^2(\mathbb{N}),
	\]
	called the Fourier transform. 
	This operator allows us to define a convolution on $L^2(X)$ by
	\begin{equation}
	f * g := \mathcal{F}^{-1}\!\big( \mathcal{F}(f) \odot \mathcal{F}(g) \big),\label{eq:conv_on_orbi}
	\end{equation}
	where $\odot$ denotes the pointwise product.
\end{thm}

\begin{proof}
	The spectrum of the Laplacian on a compact Riemannian orbifold consists of an infinite sequence of eigenvalues
	\[
	\mathrm{Spec}(\Delta_{X}) = \{ \lambda_k \}_{k \in \mathbb{N}} = \{ 0 = \lambda_0 < \lambda_1 \leq \lambda_2 \leq \cdots \},
	\]
	with $\lim_{k \to \infty} \lambda_k = \infty$. 
	Correspondingly, there exists an orthonormal basis of $L^2(X)$ consisting of smooth eigenfunctions 
	$\{ \psi_k \}_{k \in \mathbb{N}}$ of the Laplacian \citep[pp.~93--99]{Chiang1993}; see also \citep{Gordon2019}, satisfying
	\begin{equation}
		\label{eq:laplacian_ef}
		\Delta_{X} \psi_k =  \lambda_k \psi_k, 
		\qquad 
		\langle \psi_i, \psi_j \rangle = \delta_{ij}.
	\end{equation}
	The Fourier transform of a function $f \in L^2(X)$ is then given by
	\[
		\label{eq:fourier_on_orbi}	
		\mathcal{F}_{\Delta_{X}}: L^2(X) \to \ell^{2}(\mathbb{N}), 
		\qquad 
		f \mapsto \mathcal{F}_{\Delta_{X}}(f) = (\hat{f}(k))_{k \in \mathbb{N}},
		\quad 
		\hat{f}(k) := \langle f, \psi_k \rangle.
	\]
	This operator is unitary by standard Hilbert space arguments: 
	it is an isometry since it maps one orthonormal basis onto another, 
	and it is surjective due to the completeness of the eigenbasis in $L^2(X)$. 
	The spectral convolution of $f,g \in L^2(X)$ is defined by
\[	(f * g)(u) := \mathcal{F}_{\Delta_{X}}^{-1}\!\big( \mathcal{F}_{\Delta_{X}}(f) \odot \mathcal{F}_{\Delta_{X}}(g) \big)
		= \sum_{k=0}^{\infty} \hat{f}(k) \, \hat{g}(k) \, \psi_k(u).
\] Since for two sequences $a,b \in \ell^{2}(\mathbb{N})$ the pointwise product lies in $\ell^1(\mathbb{N})$ by the Cauchy-Schwarz inequality, and since every absolutely summable sequence is quadratically summable by an elementary estimate, hence $\ell^1(\mathbb{N})\subset \ell^{2}(\mathbb{N})$, the spectral product is well defined.
\end{proof}
In the following section, we illustrate the construction on an example from music theory.

\section{Example: Smoothing a Periodicity Function on the Orbifold of Dyads}
\label{section:example}
In this section we show how the logarithmic periodicity function on the chord orbifold $\mathcal{C}^2_{12}$ is smoothed by convolution with a smoothing filter.

\subsection{Chord Space}

The fundamental frequency \( f \) of a musical note is perceived on a
logarithmic scale and is represented as the pitch \( p \).
Following the common convention in music theory, we assign
\( \text{B3} = -1 \), \( \text{C4} = 0 \), and so on, with the reference
frequency \( f_0 = 261.626\,\text{Hz} \) corresponding to C4.
This yields the pitch mapping
\begin{equation}\label{eq:pitch}
	f \;\longmapsto\; p = 12 \cdot \log_2\!\left(\frac{f}{f_0}\right),
\end{equation}
where pitches are measured in semitones relative to~C4.  
Thus, $p$ is an absolute pitch value: integer values correspond to the
twelve semitone steps of the chromatic scale, while non-integer values indicate
microtonal pitches.

Cents, in contrast, do not represent absolute pitch values.
A cent is a logarithmic unit used exclusively for measuring
interval sizes, i.e.\ differences between pitches.
By definition, $100$\,cents equal one semitone, and therefore $1200$\,cents
equal one octave.
The size of an interval between two frequencies \( f_1 \) and \( f_2 \) in cents
is given by
\[
d = 1200 \cdot \log_2\!\left(\frac{f_2}{f_1}\right).
\]
If \(p_1\) and \(p_2\) denote the corresponding semitone pitches, then
\( d = 100 \,(p_2 - p_1) \).

A chord consisting of exactly \( n \) notes can be represented as a point in
\( \mathbb{R}^n \), where each coordinate corresponds to the pitch of a note.
Since the simultaneous notes of a chord are unordered, permutations of
coordinates do not change the chord, motivating the identification of chord
space with the quotient \( \mathbb{R}^n / S_n \), where the symmetric group
\( S_n \) acts by permuting coordinates.

Moreover, chord inversions are typically regarded as equivalent in music
theory. Incorporating octave equivalence yields the torus
\[
\mathbb{T}_\gamma^n = (\mathbb{R}/\gamma\mathbb{Z})^n
\cong \mathbb{R}^n / (\gamma\mathbb{Z})^n,
\]
and thus the chord space modulo ordering and octave transpositions is
\begin{equation}\label{eq:chords}
	\mathcal{C}^n_\gamma := \mathbb{T}^n_\gamma / S_n
	\;\cong\; (\mathbb{R}^n / S_n) / (\gamma\mathbb{Z})^n .
\end{equation}
In a twelve-tone system, $\gamma = 12$ corresponds to octave equivalence in
semitones.
Generally, via rescaling, it is $\mathcal{C}^n_\gamma \cong \mathcal{C}^n_\delta$, for $\gamma, \delta \in \mathbb{N}$.

\begin{exmp} \label{example:1} The space of dyads (two-note chords) is given by
	\[
		\mathcal{C}^{2}_{12} = \mathbb{T}_{12}^2 / S_2 \cong \mathbb{T}_1^2 / S_2 = \mathcal{C}^{2}_{1}.
	\]
	This space consists of all unordered pairs of pitch classes (modulo octaves).
	The fundamental domain of $\mathcal{C}_\gamma^2$ as a quotient of $\mathbb{R}^2$ as defined in Equation \eqref{eq:chords}  is given by the triangle
	\[
	T_\gamma := \left\{ (x,y) \in \mathbb{R}^2\;\middle|\; 0\leq y \leq x < \gamma\right\} .
	\] 
	However, $\mathcal{C}_1^2$ can be interpreted as a Möbius strip as follows. The Möbius strip is the non-orientable manifold with boundary viewed as an embedding in $\mathbb{R}^3$ via the parameterization
	\[
	\varrho(\alpha, r) := \begin{pmatrix}
		\cos{(2\alpha)} \cdot \left(1 + r \cos{(\alpha)}\right) \\
		\sin{(2\alpha)} \cdot \left(1 + r \cos{(\alpha)}\right) \\
		r \sin{(\alpha)}
	\end{pmatrix}, \quad \text{with } 	(\alpha,r) \in  [0, \pi) \times [-\tfrac{1}{2},\tfrac{1}{2}].
	\]
	Topologically, this corresponds to gluing $[0, \pi] \times [-\tfrac{1}{2},\tfrac{1}{2}]$ via $(0,x)\sim(\pi,-x)$.	Then the map from the fundamental domain \( T_1 \) of $\mathcal{C}_1^2$ onto the smoothly embedded Möbius strip is given by the composite map
	\[
	\varrho \circ \varphi: T_1  \to \mathbb{R}^{3}
	\] with
	\[
	\varphi(x,y) := \begin{cases}
		(\pi(x+y),x-y-\frac{1}{2}), & \text{if } x + y < 1 , \\
		( \pi(x+y-1),y-x+\frac{1}{2}), & \text{if } x + y \geq 1
	\end{cases}
	\]
	shown in Figure \ref{fig:Mapping}, which factors through $\mathcal{C}_1^2$ as a homeomorphism onto the Möbius strip. 
	\begin{figure}[tbp]
		\centering
		\def\svgwidth{\textwidth}
\begingroup%
  \makeatletter%
  \providecommand\color[2][]{%
    \errmessage{(Inkscape) Color is used for the text in Inkscape, but the package 'color.sty' is not loaded}%
    \renewcommand\color[2][]{}%
  }%
  \providecommand\transparent[1]{%
    \errmessage{(Inkscape) Transparency is used (non-zero) for the text in Inkscape, but the package 'transparent.sty' is not loaded}%
    \renewcommand\transparent[1]{}%
  }%
  \providecommand\rotatebox[2]{#2}%
  \newcommand*\fsize{\dimexpr\f@size pt\relax}%
  \newcommand*\lineheight[1]{\fontsize{\fsize}{#1\fsize}\selectfont}%
  \ifx\svgwidth\undefined%
    \setlength{\unitlength}{2151.35998535bp}%
    \ifx\svgscale\undefined%
      \relax%
    \else%
      \setlength{\unitlength}{\unitlength * \real{\svgscale}}%
    \fi%
  \else%
    \setlength{\unitlength}{\svgwidth}%
  \fi%
  \global\let\svgwidth\undefined%
  \global\let\svgscale\undefined%
  \makeatother%
  \begin{picture}(1,0.26271753)%
    \lineheight{1}%
    \setlength\tabcolsep{0pt}%
    \put(0,0){\includegraphics[width=\unitlength,page=1]{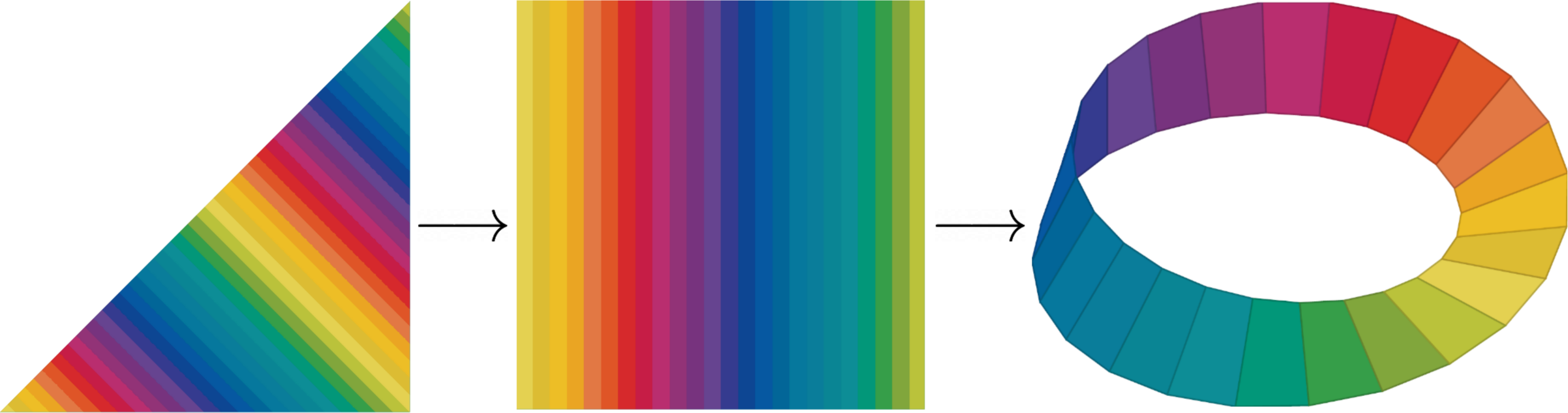}}%
    \put(0.28640992,0.12648256){\color[rgb]{0,0,0}\makebox(0,0)[lt]{\lineheight{1.25}\smash{\begin{tabular}[t]{l}$\varphi$\end{tabular}}}}%
    \put(0.61777558,0.12648256){\color[rgb]{0,0,0}\makebox(0,0)[lt]{\lineheight{1.25}\smash{\begin{tabular}[t]{l}$\varrho$\end{tabular}}}}%
  \end{picture}%
\endgroup%

		\caption{The mapping $\varrho\circ\varphi$ between the fundamental domain $T_1$ of $\mathcal{C}^2_1$ and the Möbius strip as it is classically shown embedded in $\mathbb{R}^{3}$.
		}
		\label{fig:Mapping}
	\end{figure}
\end{exmp}
Even though $\mathbb{R}^n$  and the torus $\mathbb{T}_{\gamma}^n$ are smooth manifolds, neither one of the quotients $\mathbb{R}^n/S_n$ and $\mathcal{C}^n_\gamma$ are manifolds. In fact, they are orbifolds \citep{Tymoczko2011,Himpel2022}. In particular, $\mathcal{C}^{2}_{12}$ is, according to Example \ref{example:2}, an effective smooth orbifold whose underlying topological space is homeomorphic to the Möbius strip, as explained in Example~\ref{example:1}.

\subsection{Laplacian Eigenfunctions on $\mathcal{C}^{2}_{12}$}
\label{section:laplacian_eigenfunctions_on_the_moebius_strip}

Any real function on $\mathbb{T}^{2}_{12}$ may be written as a function  
$f \colon \mathbb{R}^{2} \to \mathbb{R}$ satisfying the periodicity condition  
$f(x) = f(x + 12 \sum_{i=1}^{2} k_i e_i)$ for $k_{1}, k_{2} \in \mathbb{Z}$.  
Eigenfunctions of the Laplacian therefore satisfy
\[
\Delta_{\mathbb{T}^{2}_{12}} \psi_k
=
\frac{\partial^{2}\psi_k}{\partial x^{2}}
+
\frac{\partial^{2}\psi_k}{\partial y^{2}}
=
\lambda_k \psi_k,
\]
compare Equation~\eqref{eq:laplacian_ef}.  
By \citep{Zelditch2017}, an orthonormal Fourier basis of eigenfunctions is given by
\[
\psi_{c_{1},c_{2}}(x,y)
=
\exp\!\left(
\frac{2\pi i}{12}(c_{1} x + c_{2} y)
\right),
\qquad
c_{1}, c_{2} \in \mathbb{Z},
\]
with eigenvalues
\[
\lambda_{c_{1},c_{2}}
=
\left( \frac{2\pi}{12} \right)^{2} (c_{1}^{2} + c_{2}^{2}).
\]

Since $\mathcal{C}^{2}_{12}$ is obtained as the orbifold quotient of $\mathbb{T}^{2}_{12}$ by the $S_{2}$ action, only eigenfunctions invariant under the coordinate swap  
$(x,y) \mapsto (y,x)$ descend to the quotient.  
The symmetrising operator
\[
(R f)(x,y)
:=
\frac{1}{2}\big( f(x,y) + f(y,x) \big)
\]
is an orthogonal projection onto the subspace of symmetric functions.  
The symmetrised eigenfunctions $R(\psi_{c_{1},c_{2}})$ therefore form a complete orthonormal basis of $L^{2}(\mathcal{C}^{2}_{12})$ that encodes spatial variations of functions over the dyad orbifold and should not be mistakenly interpreted as acoustic waveforms that can be heard. Selected examples of eigenfunctions can be seen in Figure \ref{fig:ex_eigenfunctions}.

\begin{figure}[tbp]
	\centering
	\resizebox{0.5\textwidth}{!}{
	\begin{subfigure}{0.32\textwidth}
		\centering
		\includegraphics[width=\linewidth]{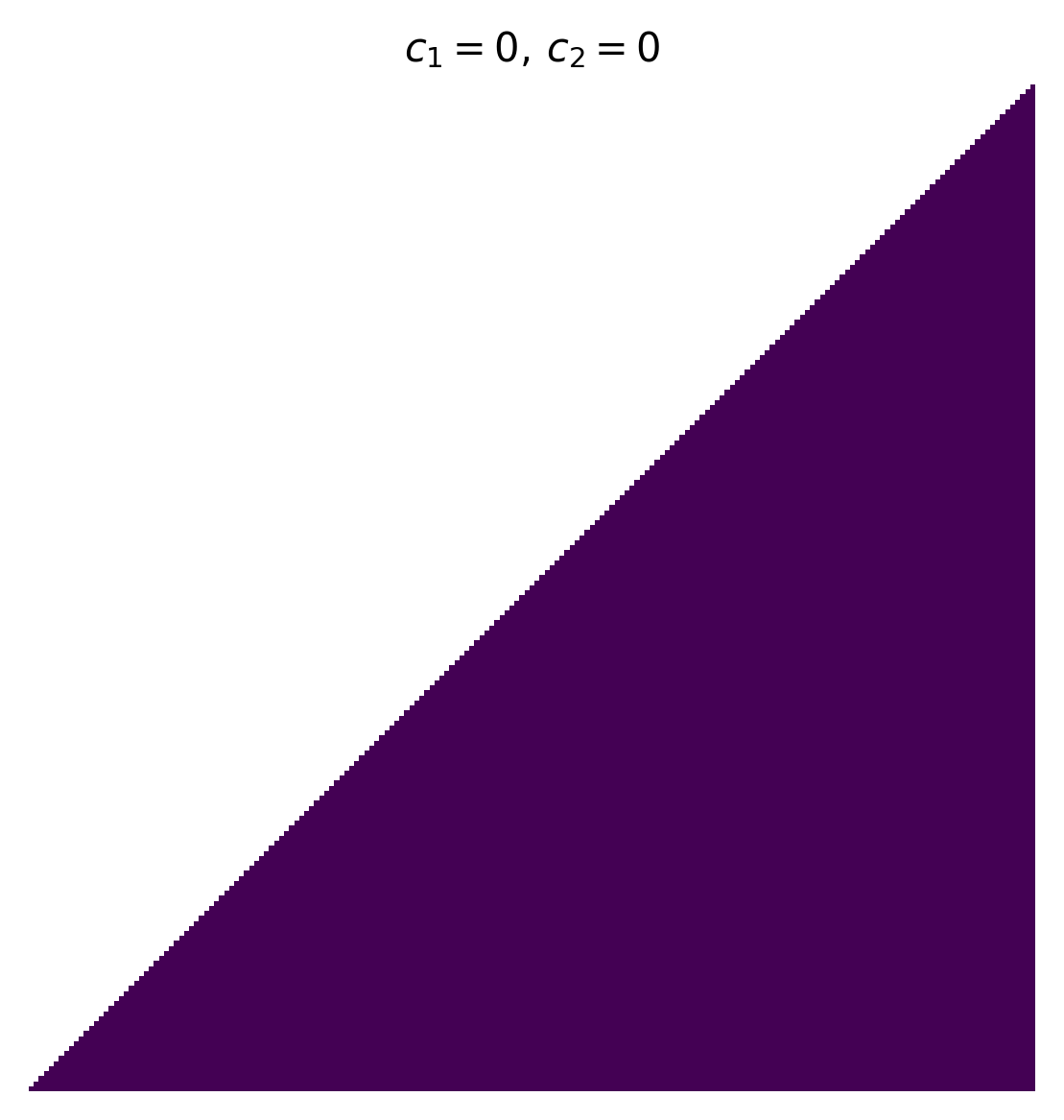}
		\label{fig:ef_1}
	\end{subfigure}
	\hfill
	\begin{subfigure}{0.32\textwidth}
		\centering
		\includegraphics[width=\linewidth]{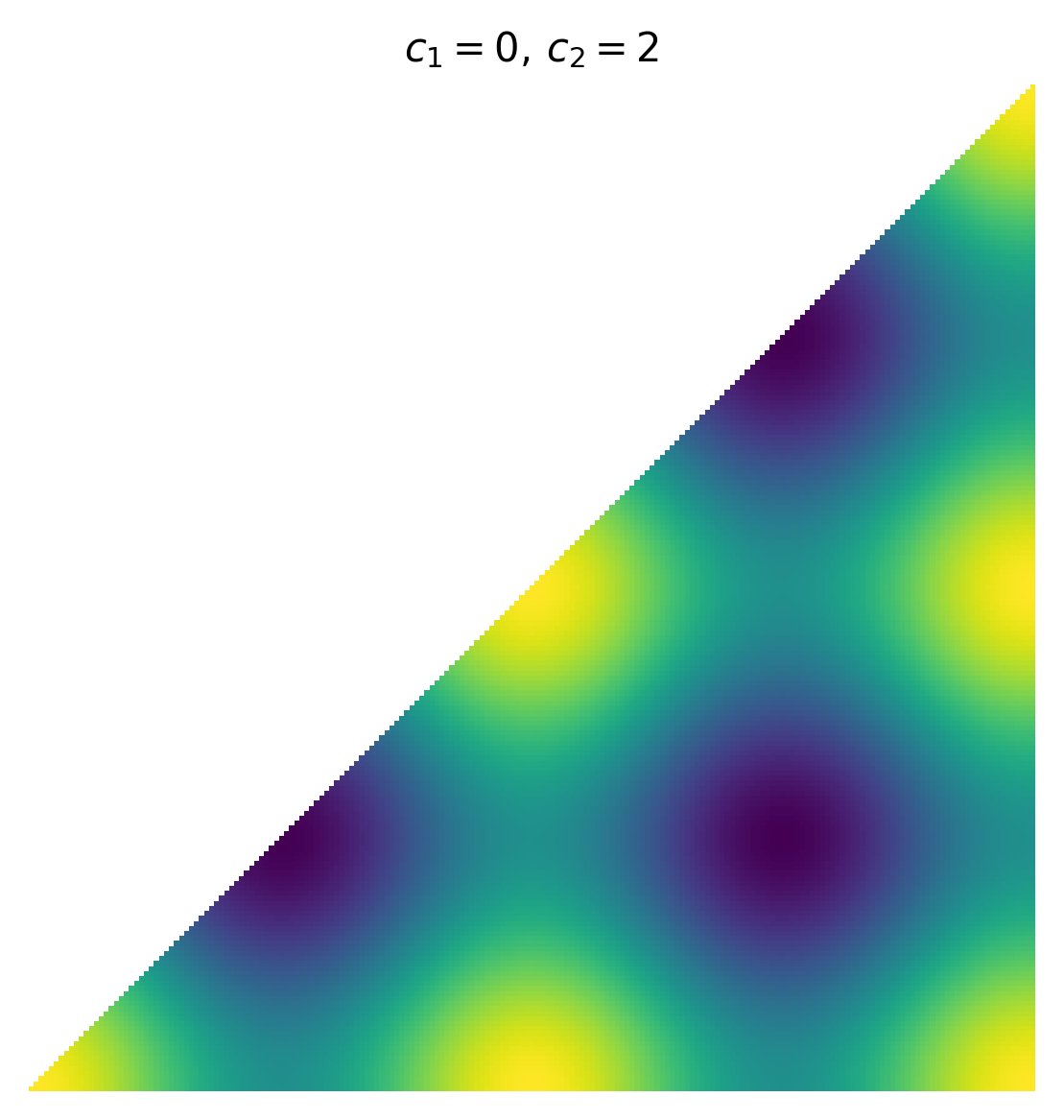}
		\label{fig:ef_2}
	\end{subfigure}
	\hfill
	\begin{subfigure}{0.32\textwidth}
		\centering
		\includegraphics[width=\linewidth]{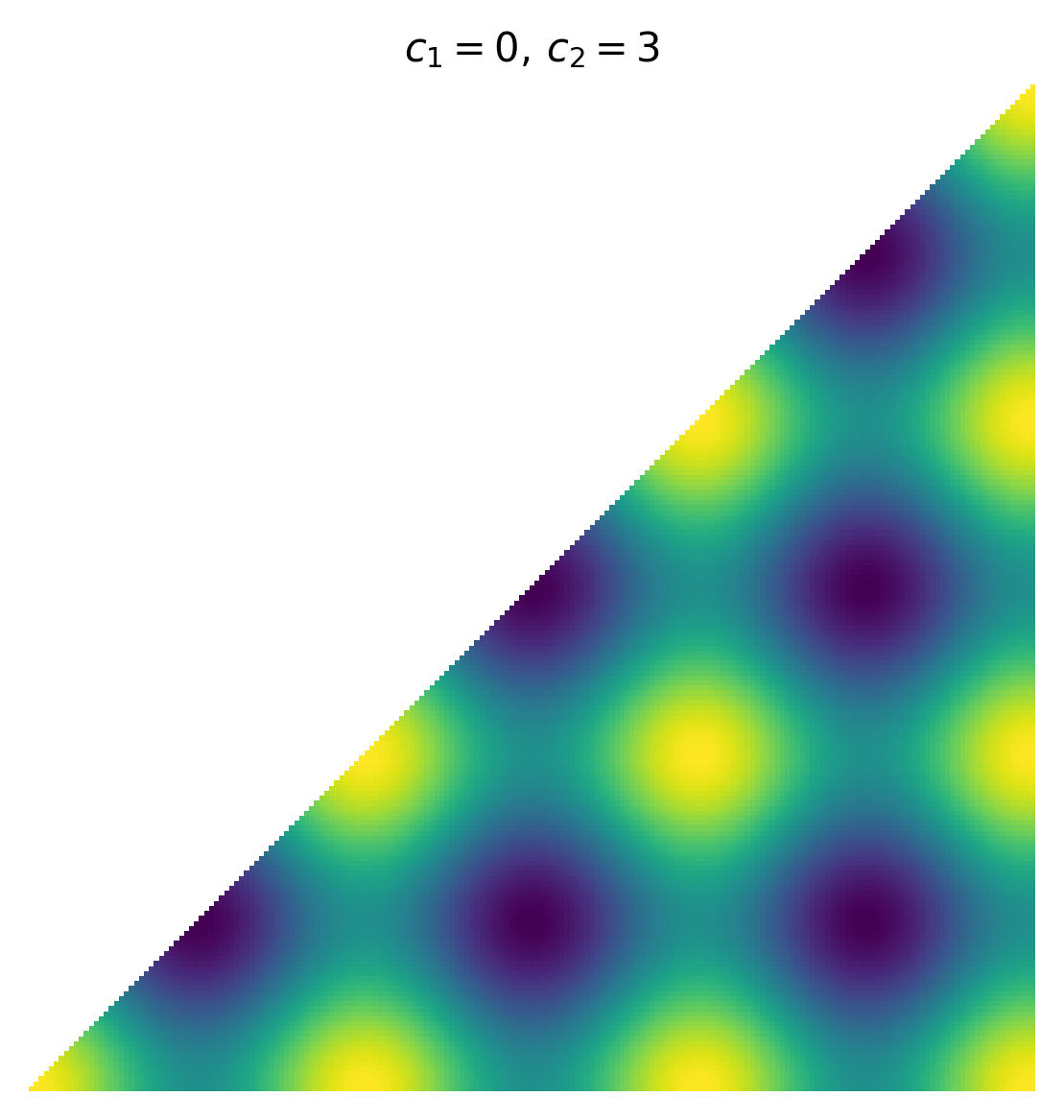}
		\label{fig:ef_3}
	\end{subfigure}
	}
	\vspace{0.5em}
	\resizebox{0.5\textwidth}{!}{
	\begin{subfigure}{0.32\textwidth}
		\centering
		\includegraphics[width=\linewidth]{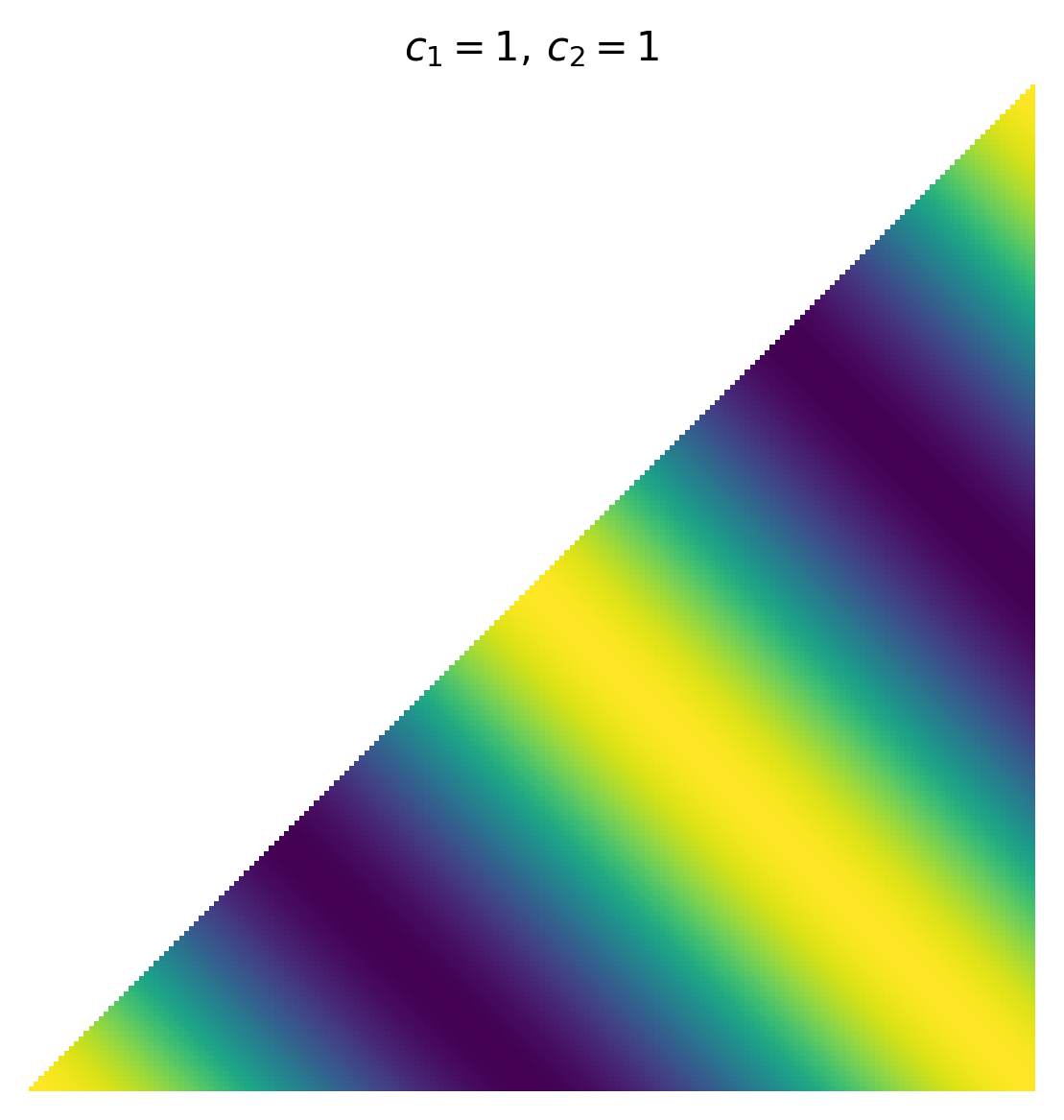}
		\label{fig:ef_4}
	\end{subfigure}
	\hfill
	\begin{subfigure}{0.32\textwidth}
		\centering
		\includegraphics[width=\linewidth]{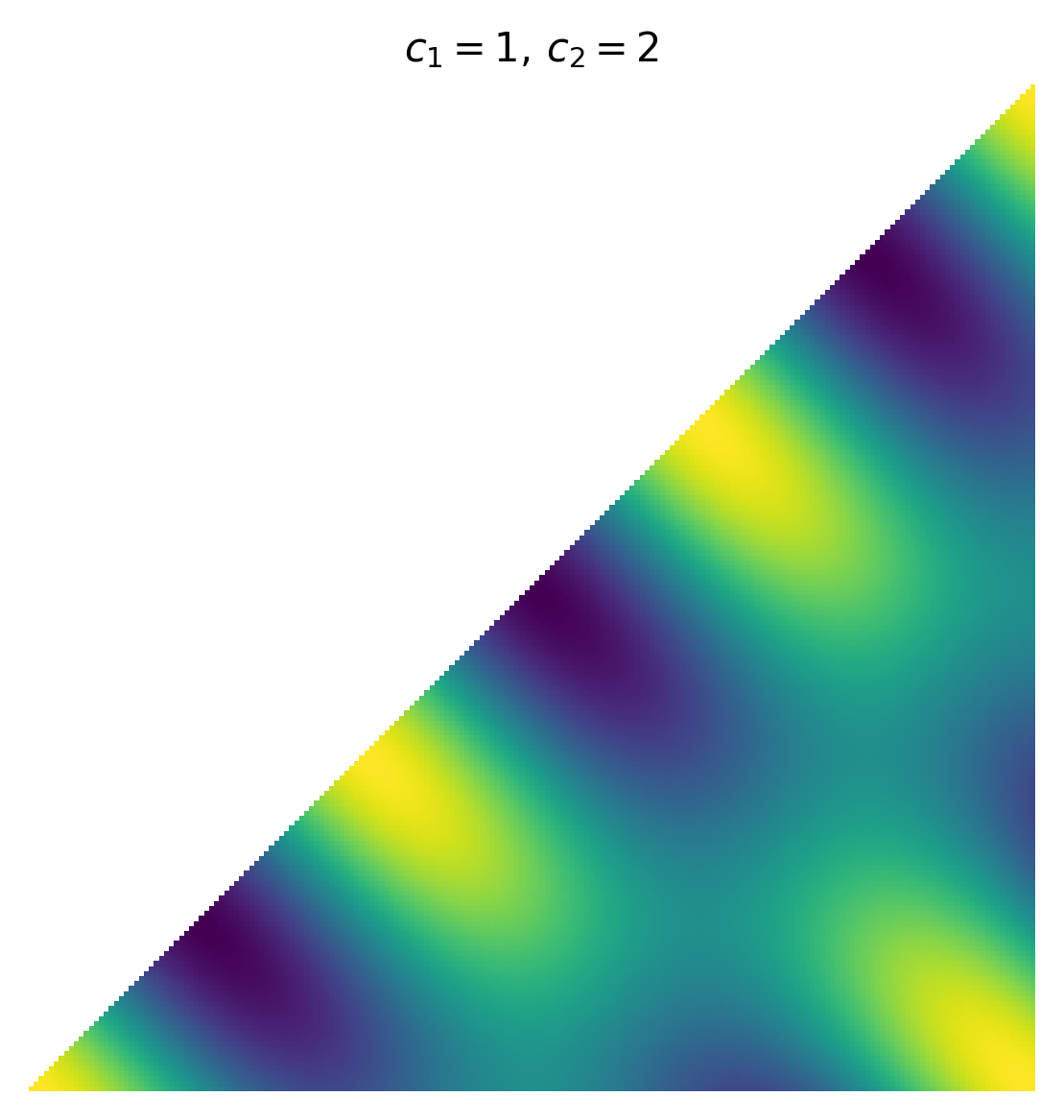}
		\label{fig:ef_5}
	\end{subfigure}
	\hfill
	\begin{subfigure}{0.32\textwidth}
		\centering
		\includegraphics[width=\linewidth]{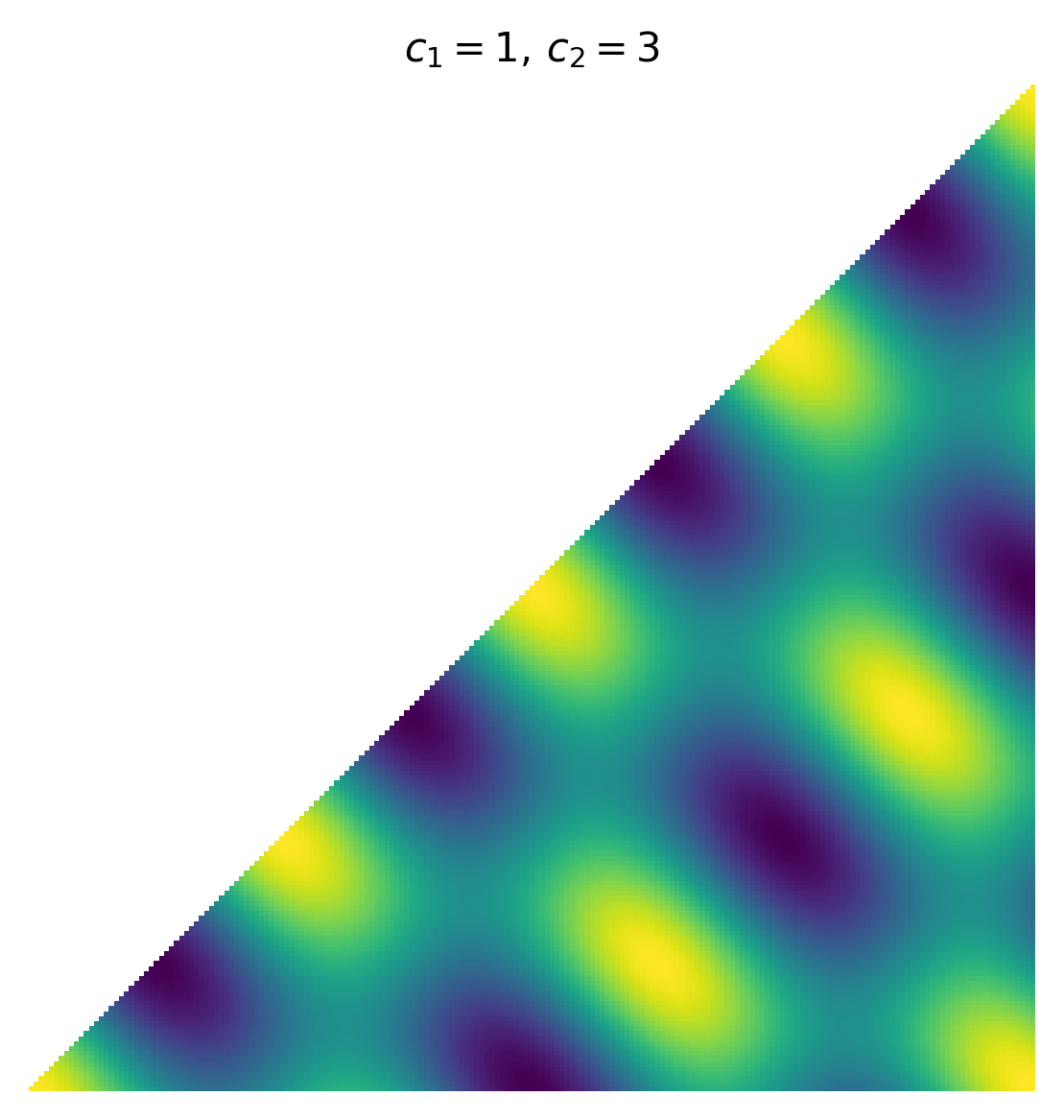}
		\label{fig:ef_6}
	\end{subfigure}
	}
	\vspace{0.5em}
	\resizebox{0.5\textwidth}{!}{
	\begin{subfigure}{0.32\textwidth}
		\centering
		\includegraphics[width=\linewidth]{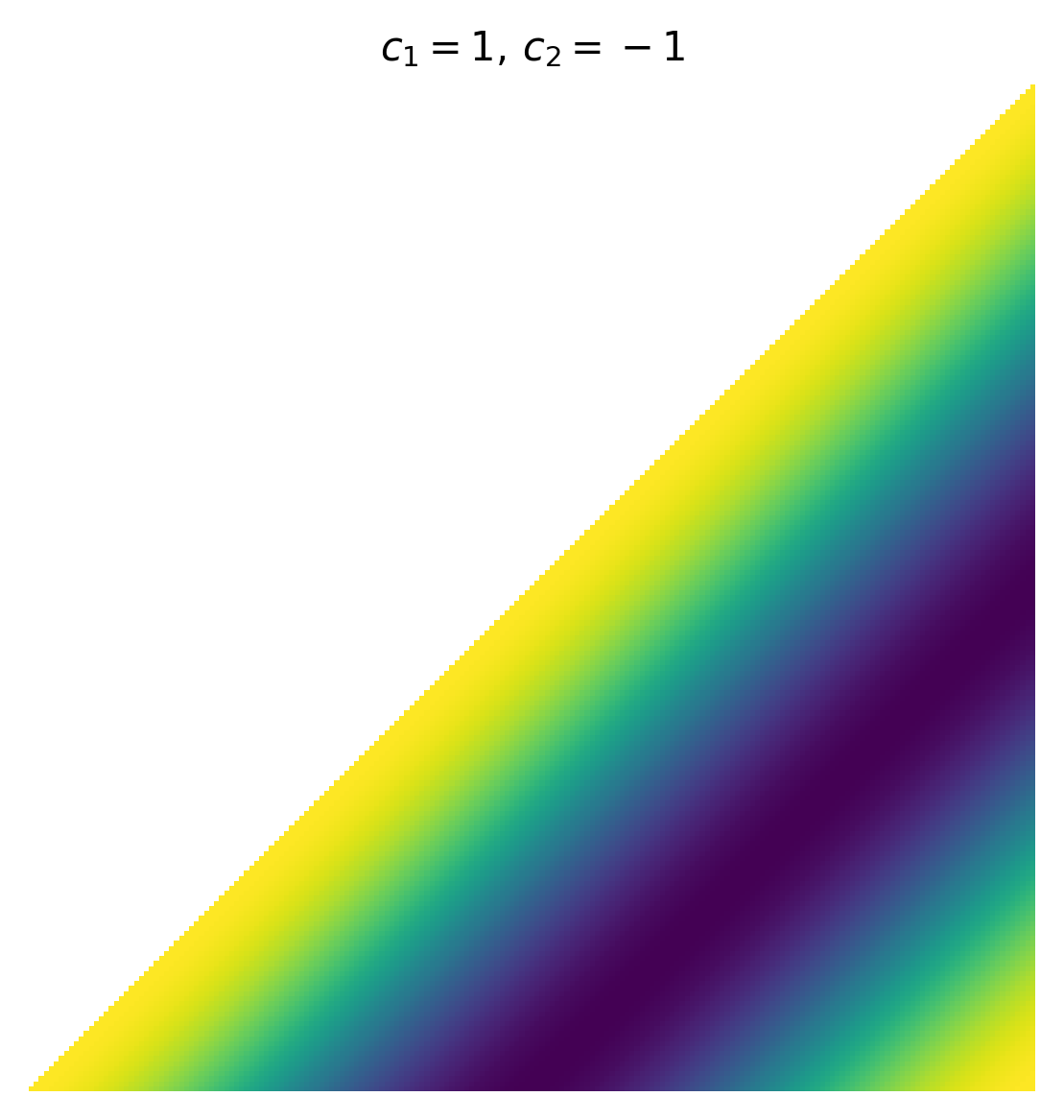}
		\label{fig:ef_7}
	\end{subfigure}
	\hfill
	\begin{subfigure}{0.32\textwidth}
		\centering
		\includegraphics[width=\linewidth]{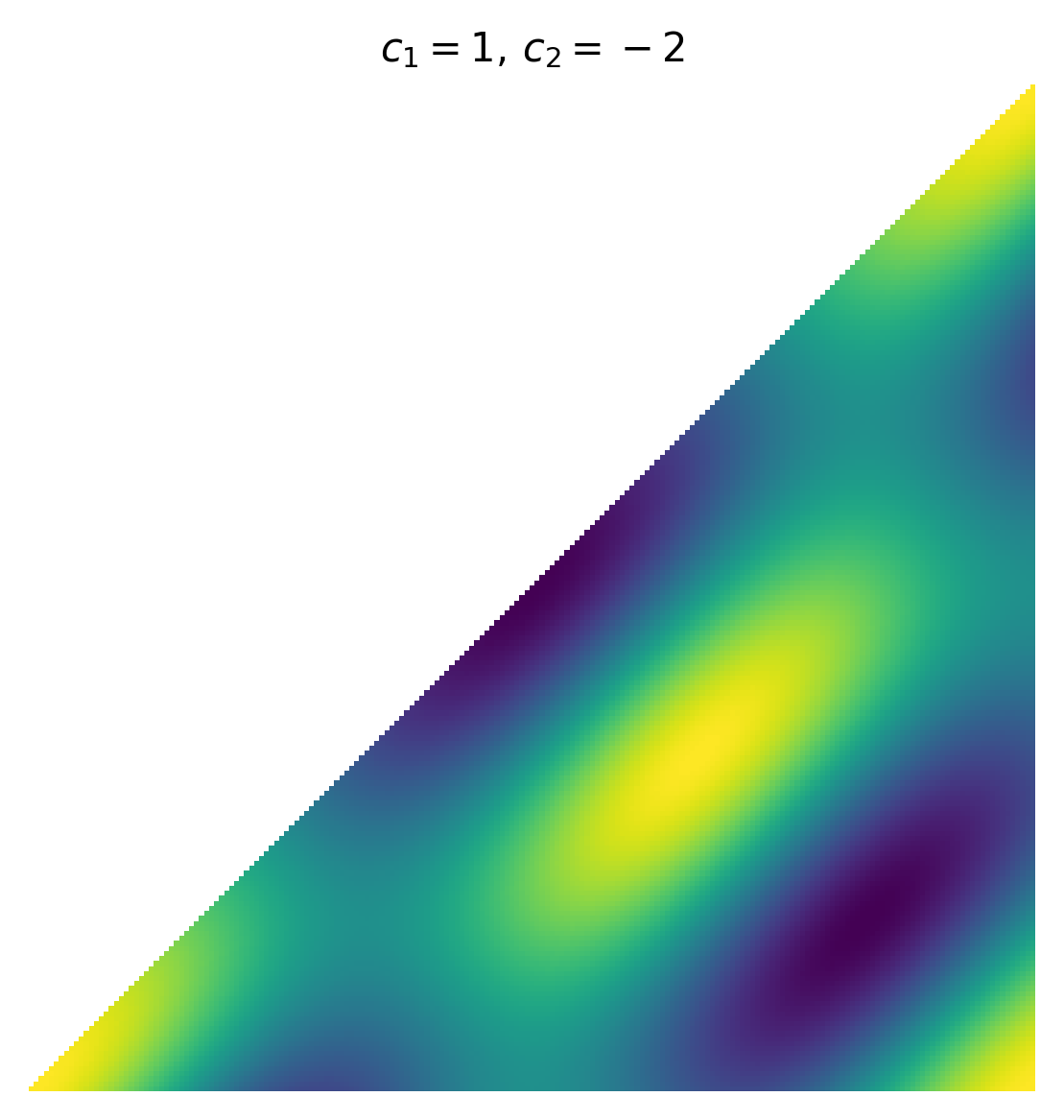}
		\label{fig:ef_8}
	\end{subfigure}
	\hfill
	\begin{subfigure}{0.32\textwidth}
		\centering
		\includegraphics[width=\linewidth]{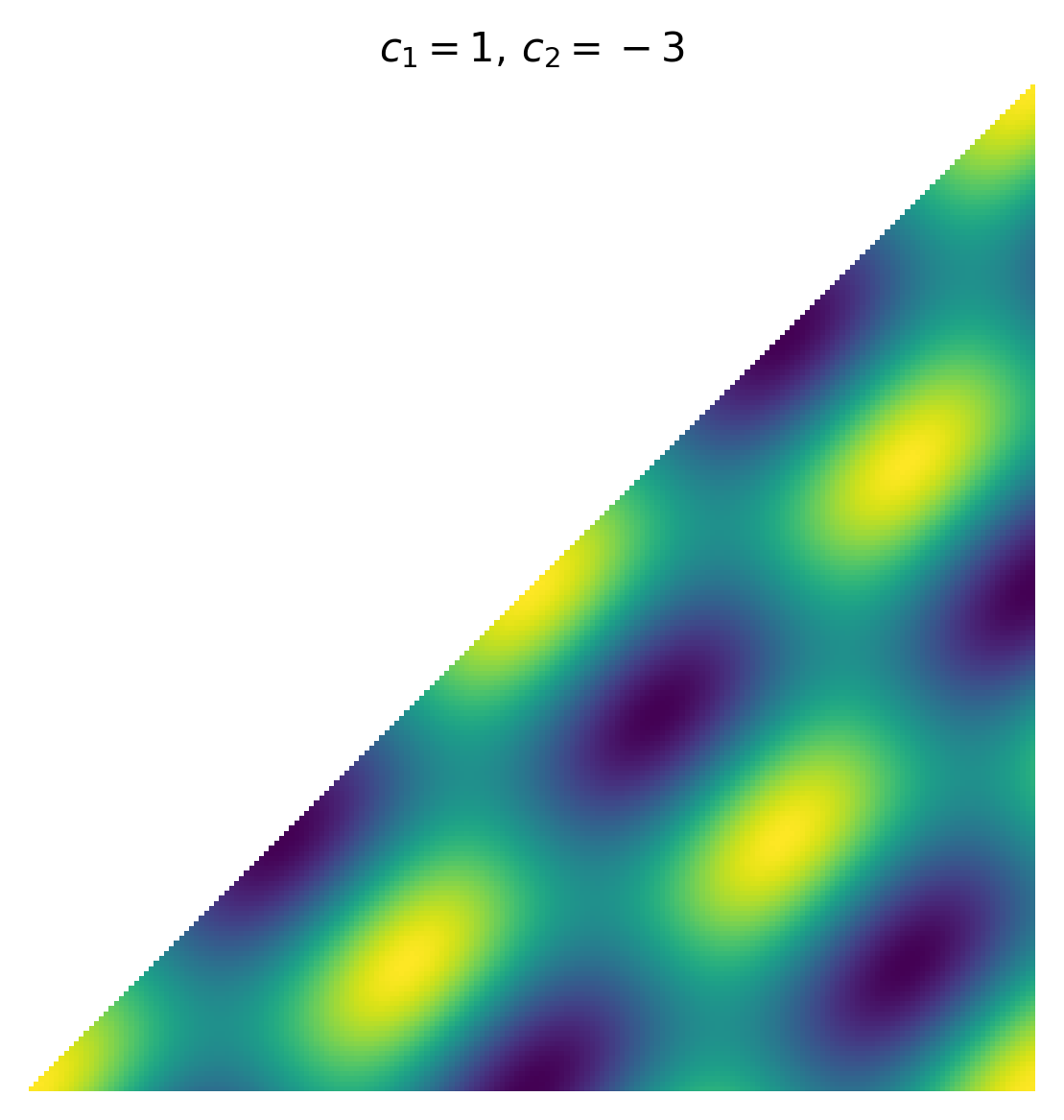}
		\label{fig:ef_9}
	\end{subfigure}
	}
	\caption{Examples of the real part of some eigenfunctions on $\mathcal{C}^2_{12}$.}
	\label{fig:ex_eigenfunctions}
\end{figure}

\subsection{Smoothing a Periodicity Function}
\label{section:smoothing_a_periodicity_function}

Two frequencies $f_{1} \leq f_{2}$ are in a rational frequency ratio if
\begin{equation}
	\label{eq:rational_relationship}
	\frac{f_{2}}{f_{1}} = \frac{a}{b},
\end{equation}
where $a, b \in \mathbb{N}$ and $\gcd(a,b)=1$.
In music theory and psychoacoustics, intervals characterised by small integers
$a$ and $b$ are regarded as more consonant because they minimise interference
effects between partials, the individual frequency components of a sound
\citep{Sethares2005}.

Stolzenburg \citep{Stolzenburg2015} quantifies the consonance between two
frequencies $f_1$ and $f_2$ by the denominator~$b$ of their reduced
ratio~\eqref{eq:rational_relationship}, a quantity he refers to as periodicity.
This periodicity equals the length of the joint period of the two sinusoidal
frequencies, corresponding to $b$ periods of~$f_1$, and is psychoacoustically
meaningful because the auditory system tends to judge sounds as more consonant
when their combined waveform repeats after only a short time span
\citep{Stolzenburg2015,Himpel2022}.

Human listeners, however, do not perceive frequency linearly.
Pitch, as defined in Equation~\eqref{eq:pitch}, is measured in semitones
and therefore represents an absolute logarithmic pitch scale.
In contrast, cents measure only differences between pitches:
$100$\,cents correspond to one semitone and $1200$\,cents to one octave.
Because interval size is perceived logarithmically, it is natural to express
Stolzenburg's periodicity measure in logarithmic units and define the logarithmic periodicity of the pitch pair $(p_1,p_2)$ associated with 
frequencies $f_1,f_2$ with rational ratio by
\[
	P(p_{1},p_{2}) = \log_2(b),
\]
where $b$ is as in~\eqref{eq:rational_relationship}.
Small values of $P$ correspond to simple (consonant) pitch
relationships, whereas larger values indicate more complex or dissonant ones.

Importantly, closely spaced pitches cannot be discriminated reliably by
human listeners. Hence, the perceived complexity of an interval remains constant
as long as the pitch difference lies within one discrimination threshold.
This threshold is defined by the just noticeable difference (JND), the
smallest pitch change detectable in a comparison task, which is
typically given in cents.
Since pitches $p$ are measured in semitones, a JND of $c$ cents corresponds to a
pitch difference of $c/100$ semitones.
JNDs are estimated psychophysically at the point where changes in a stimulus
dimension become reliably perceptible \citep{Bausenhart2018}, thereby
delineating the region within which frequency differences do not produce
perceptually distinguishable changes in interval complexity.

Consequently, a psychoacoustic measure of complexity based on $P$
should remain approximately constant for pitch variations that lie within one
JND. This motivates the following JND-based measure of interval complexity for
arbitrary pairs of pitches:
\[
P_{\text{JND}}(p_1,p_2) =
\min\!\left\{
P(p_1,p)
\;\middle|\;
|p - p_2| \le \frac{\text{JND}}{100},
\;
\frac{f(p)}{f(p_1)} \in \mathbb{Q}
\right\},
\]
where $f(p)=f_0\,2^{p/12}$ denotes the frequency associated with the pitch~$p$.

Since the interval complexity depends only on the pitch difference between 
$p_1$ and $p_2$, we express it as a function of the interval 
$d = 100\,(p_2 - p_1)$ measured in cents and define
\[
P_{\text{JND}}(d)
:= P_{\text{JND}}(p_1,p_2).
\]
The resulting function $P_{\text{JND}}(d)$ is shown in 
Figure~\ref{fig:periodicity_function}.
\begin{figure}[tbp]
	\centering
	\begin{subfigure}{0.48\textwidth}
		\centering
		\includegraphics[width=\linewidth]{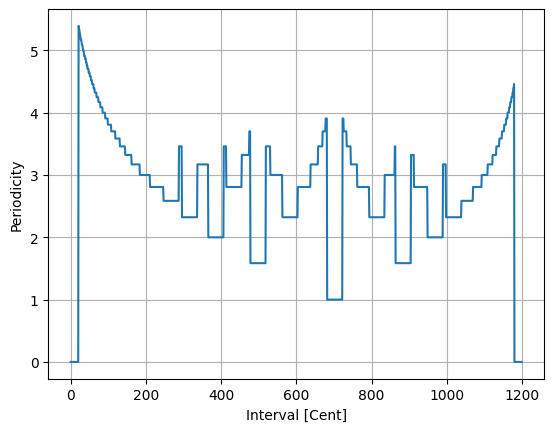}
		\caption{Periodicity function.}
		\label{fig:periodicity_function}
	\end{subfigure}
	\hfill
	\begin{subfigure}{0.48\textwidth}
		\centering
		\includegraphics[width=\linewidth]{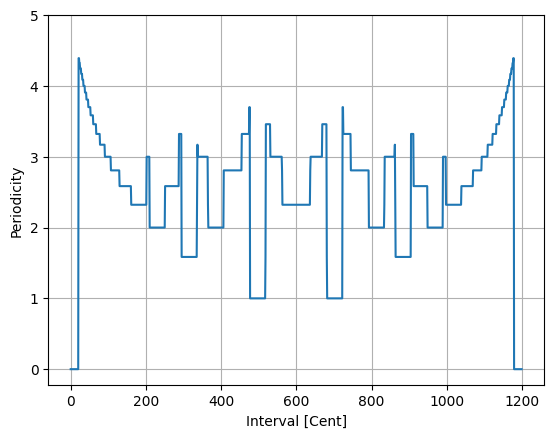}
		\caption{Symmetrized function.}
		\label{fig:symm_periodicity_function}
	\end{subfigure}
	\caption{Visualization of the periodicity function and its symmetrized version.}
	\label{fig:periodicity_combined}
\end{figure}
Behavioral and psychoacoustic studies demonstrate that just-noticeable differences for pitch are not fixed quantities but depend on multiple stimulus- and listener-related factors, including frequency region, sound level, timbre, and individual auditory sensitivity. Under controlled laboratory conditions employing pure or harmonic tones, frequency difference limens can reach very small values; however, substantial interindividual variability persists even in such idealized settings \citep{Micheyl2012}. 

Importantly, several experimental paradigms deliberately probe pitch deviations in the range of approximately $20-30$ cents, a region that lies near the transition between reliable and unreliable discrimination for many listeners and exhibits substantial interindividual variability \citep{Moeller2018}. In multisensory or timbrally complex contexts, such deviations may fall at or below reliable discriminability, as demonstrated by findings that even differences of $20-40$ cents can become difficult to detect when pitch judgments are biased by timbre or contextual factors \citep{Vurma2014_timbre_pitch_shift}.

In accordance with these findings, we adopt a representative pitch JND of $20$ cents as a
conservative yet perceptually meaningful tolerance level. This value lies well within the
empirically documented range of pitch discrimination thresholds, while remaining large
enough to account for variability across listeners, timbres, and listening conditions. The
periodicity function is therefore defined such that interval ratios falling within a
$\pm 20$-cent tolerance zone are assigned the value of the simplest rational approximation
within this perceptual neighborhood. In this way, the function reflects the fact that the
auditory system does not resolve the fine-grained discontinuities inherent to the raw
logarithmic periodicity measure.

Under the assumption that musically equivalent chords under inversion exhibit the same
degree of consonance, the periodicity function must furthermore be rendered invariant under
interval inversion. A resulting symmetrization is therefore obtained by
assigning to each interval the minimum of its periodicity value and that of its inversion,
thereby ensuring that both representatives of an inversion pair share the same consonance
value:
\[
P_{\text{JND}}^{+}(d) = \min\!\left\{
P_{\mathrm{JND}}(d),
\;
P_{\mathrm{JND}}(1200 - d)
\right\}.
\]
The resulting function $P^+_{\text{JND}}(d)$ is shown in 
Figure~\ref{fig:symm_periodicity_function}.

Although the periodicity of a dyad ultimately depends only on the interval
between its two pitches, this reduction to a one-dimensional quantity arises
from projecting the two-dimensional chord space $\mathcal{C}^2_{12}$ onto the
interval axis. The periodicity function we smooth is therefore not originally a
function on pitch differences, but a function on the dyad orbifold itself, obtained by
pulling back the interval-based periodicity to $\mathcal{C}^2_{12}$.

This distinction is essential: the Möbius topology of $\mathcal{C}^2_{12}$
encodes inversional symmetry and the non-orientability of dyad space, both of
which are lost under the one-dimensional interval parametrisation.
Spectral convolution, however, is intrinsically geometric and depends on the
Laplace eigenfunctions of the underlying space. Smoothing the periodicity function
in the full two-dimensional orbifold geometry therefore respects all symmetries
of dyad space, whereas a one-dimensional smoothing in interval coordinates would
ignore its topological and spectral structure.

The symmetrised periodicity function $P^+_{\text{JND}}$ extends canonically to the
dyad orbifold $\mathcal{C}^{2}_{12}$, yielding a well-defined,
inversion-invariant function
\begin{equation}
	\label{eq:function_to_smooth}
	P_{\text{JND}}^{+} : \mathcal{C}^2_{12} \to \mathbb{R}.
\end{equation}
The graph shown in Figure~\ref{fig:symm_periodicity_function} corresponds to a
one-dimensional section of this function, namely its restriction to the straight
line in $\mathcal{C}^2_{12}$ orthogonal to the first angle bisector and passing
through the centre of the domain, as visualised in Figure~\ref{fig:p_function_m}.
\begin{figure}[tbp]
	\centering
	\begin{subfigure}{0.32\textwidth}
		\centering
		\includegraphics[width=\linewidth]{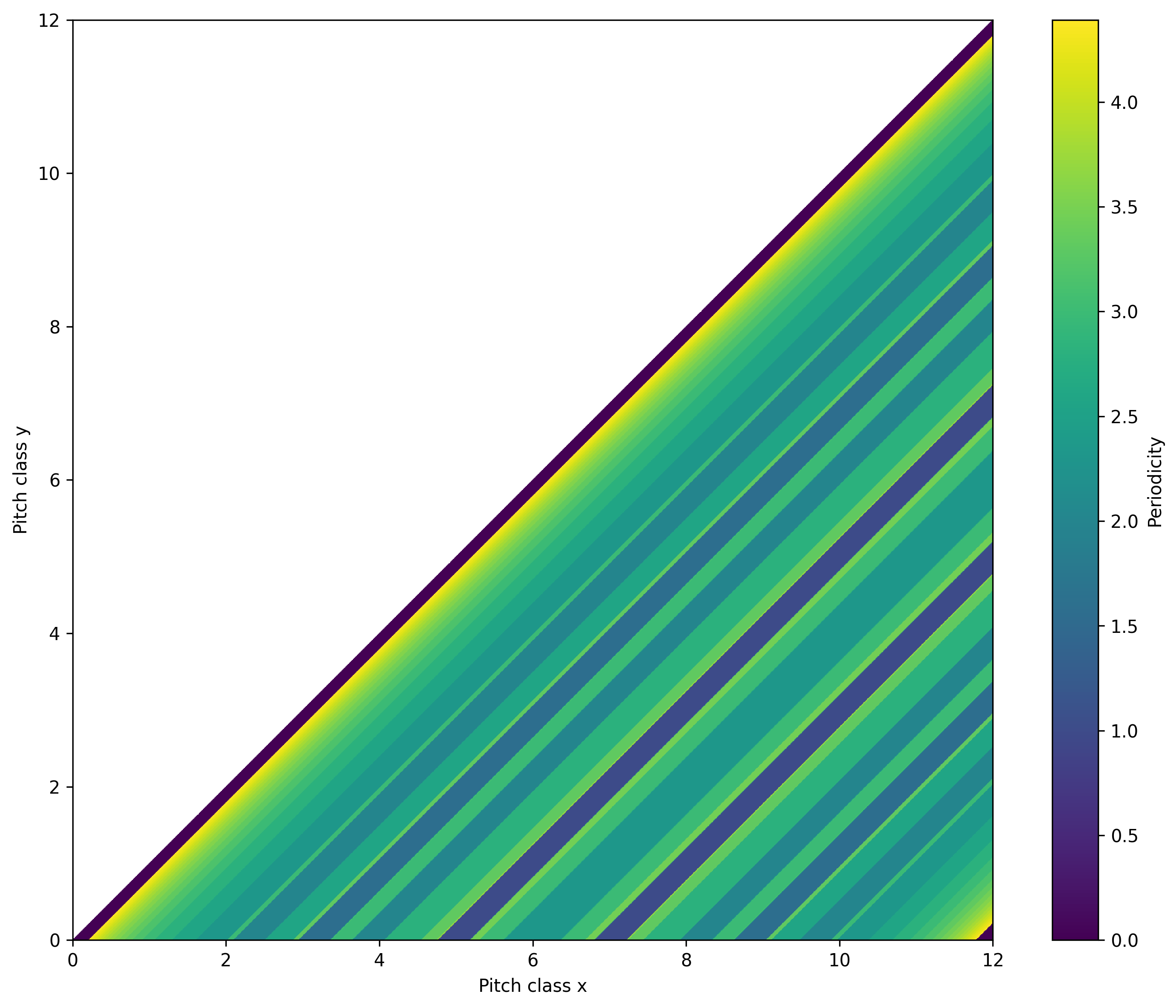}
		\caption{\small Periodicity function.}
		\label{fig:p_function_m}
	\end{subfigure}
	\hfill
	\begin{subfigure}{0.32\textwidth}
		\centering
		\includegraphics[width=\linewidth]{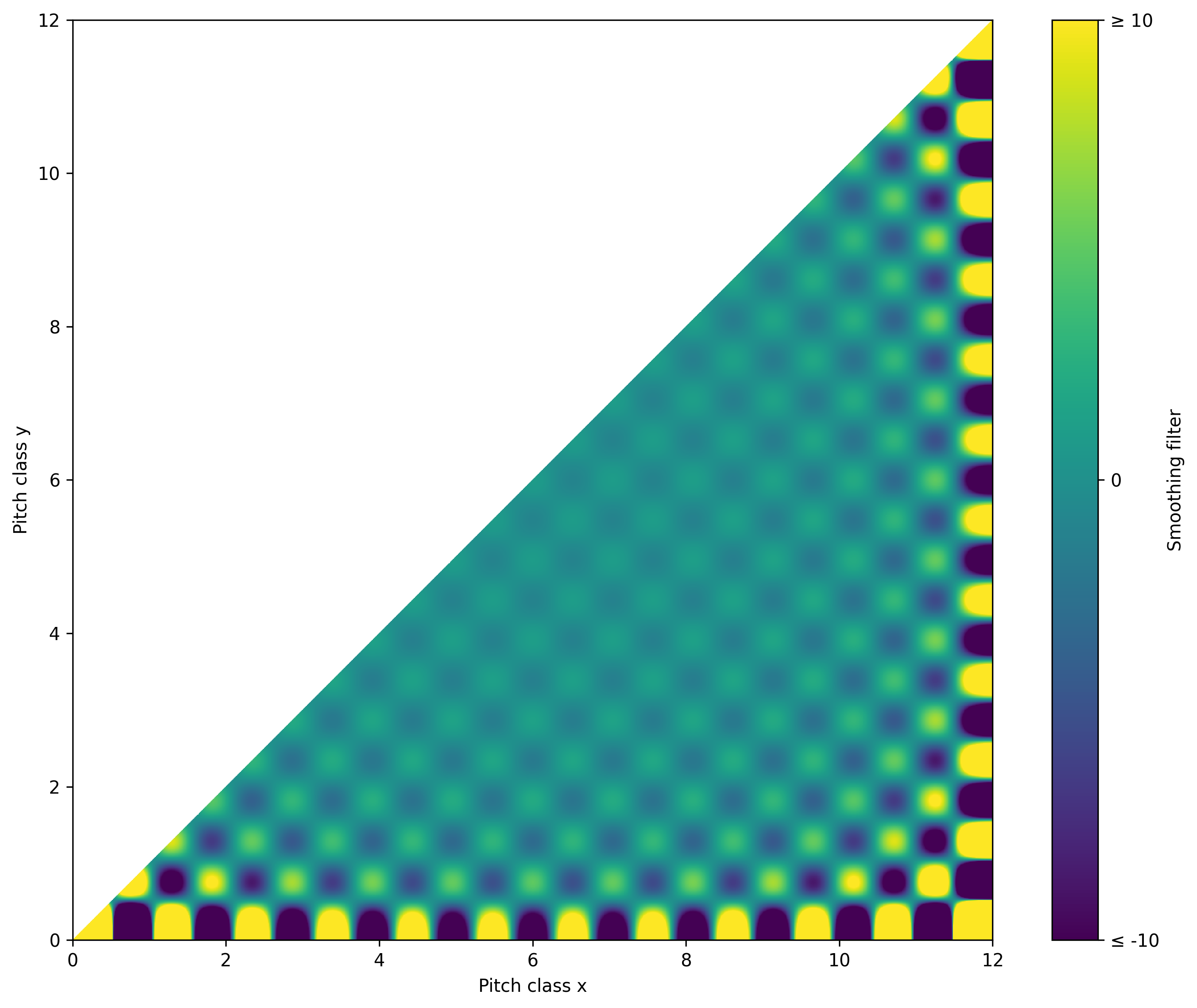}
		\caption{\small Smoothing filter.}
		\label{fig:kernel}
	\end{subfigure}
	\hfill
	\begin{subfigure}{0.32\textwidth}
		\centering
		\includegraphics[width=\linewidth]{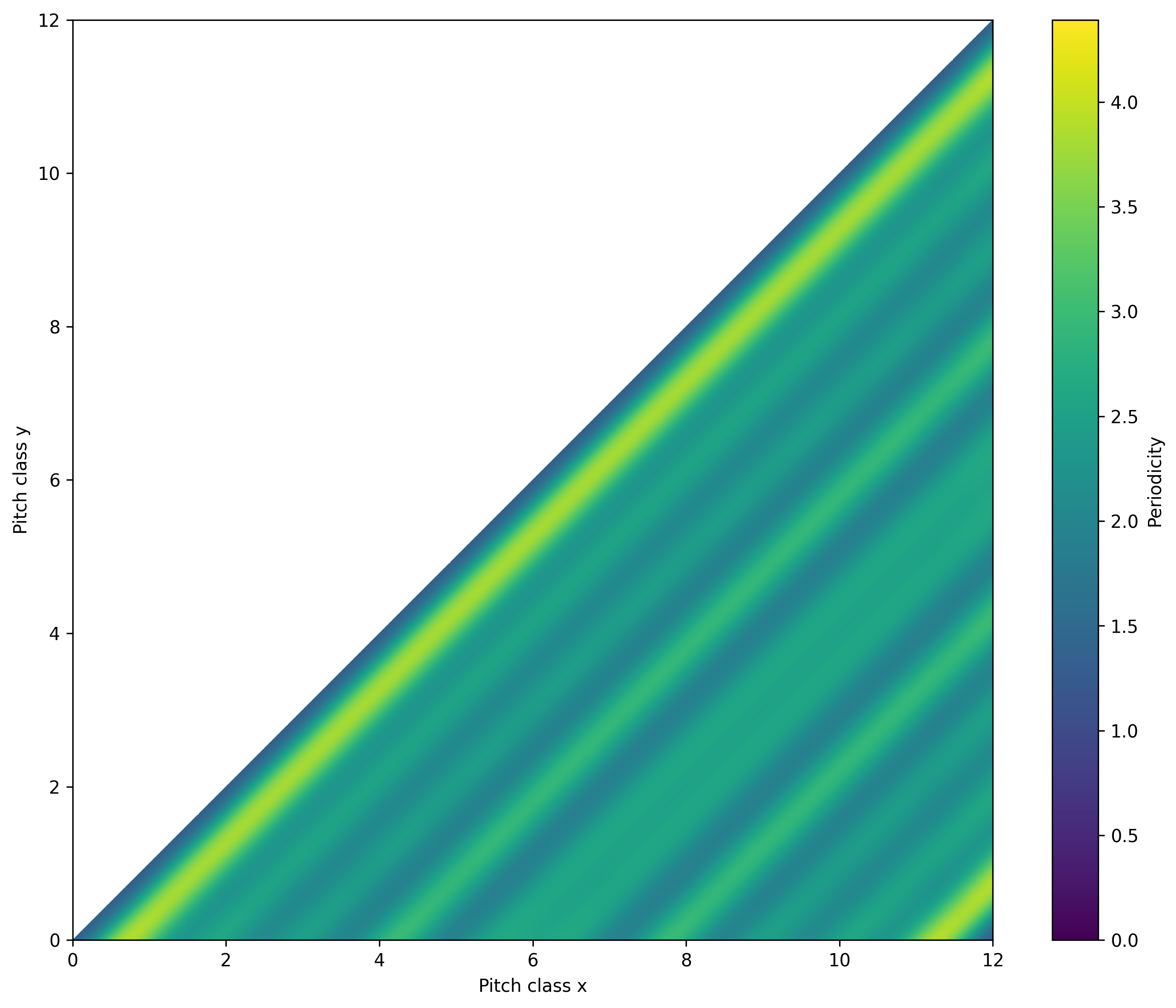}
		\caption{\small Smoothed function.}
		\label{fig:p_function_m_smoothed}
	\end{subfigure}
	\caption{(a) The symmetrized logarithmic periodicity function on $\mathcal{C}^2_{12}$, (b) the used smoothing filter, and (c) the result of the convolution.}
	\label{fig:periodicity_kernel_combined}
\end{figure}
We interpret the gradual loss of precision in auditory processing as a smoothing effect, and the continuous coding of roughness as evidence that consonance and dissonance are perceived gradually  \citep{McDermott2008}. Thereby the quality of roughness arises from various beatings between slightly different sine-wave components \citep{Himpel2022}. These assumptions motivate a smooth modeling of the periodicity function that can be implemented using the spectral convolution of the periodicity function with a smoothing filter.
The smoothing filter 
\begin{equation}
\label{eq:smoothing_filter}
g_n(u) = \sum_{k = 0}^{n} \psi_k(u)
\end{equation} is fully determined by its Fourier coefficients
\[
\hat{g}_n(k) = \left< g_n , \psi_k \right> := \begin{cases}
	1, & k \leq n, \\
	0, & k > n.
\end{cases}
\] De facto, convolution with this filter means that all coefficients up to $n$ are retained and the larger ones are discarded, which can be interpreted as a low-pass filter and the induced operator $f \mapsto f * g_n$ coincides with the orthogonal projection onto the subspace spanned by the first $n$ eigenfunctions
\[
V_n := \text{span}\{\psi_0,...,\psi_n \} \subseteq L^{2}(\mathcal{C}^2_{12}).
\] The parameter $n$ controls the spectral resolution of the smoothing: smaller
values lead to excessive smoothing that blurs genuine differences in consonance,
whereas larger values admit increasingly fine distinctions and reintroduce
oscillatory artefacts near the sharp transitions of the JND-based periodicity
function. In particular, overly large $n$ give rise to overshoots and negative
values, thereby violating a basic semantic constraint, as negative periodicity
has no meaningful interpretation. The chosen value $n=529$ provides a moderate
level of smoothing that preserves the global step structure while selectively
regularizing transitions between neighbouring regions, avoiding nonphysical
oscillations at steep boundaries.

The smoothed function is obtained by our convolution \eqref{eq:conv_on_orbi} of the periodicity function \eqref{eq:function_to_smooth} and the smoothing filter \eqref{eq:smoothing_filter}:
\[
(P_{\text{JND }}^{+,s})(u) = (P_{\text{JND }}^{+} * g_{529})(u).
\] The functions are shown in Figures \ref{fig:p_function_m}, \ref{fig:kernel} and \ref{fig:p_function_m_smoothed}. The comparison of the symmetrized and smoothed logarithmic periodicity function can be found in Figure \ref{fig:comparison}.
\begin{figure}[tbp]
    \centering
    \includegraphics[width=1.0\textwidth]{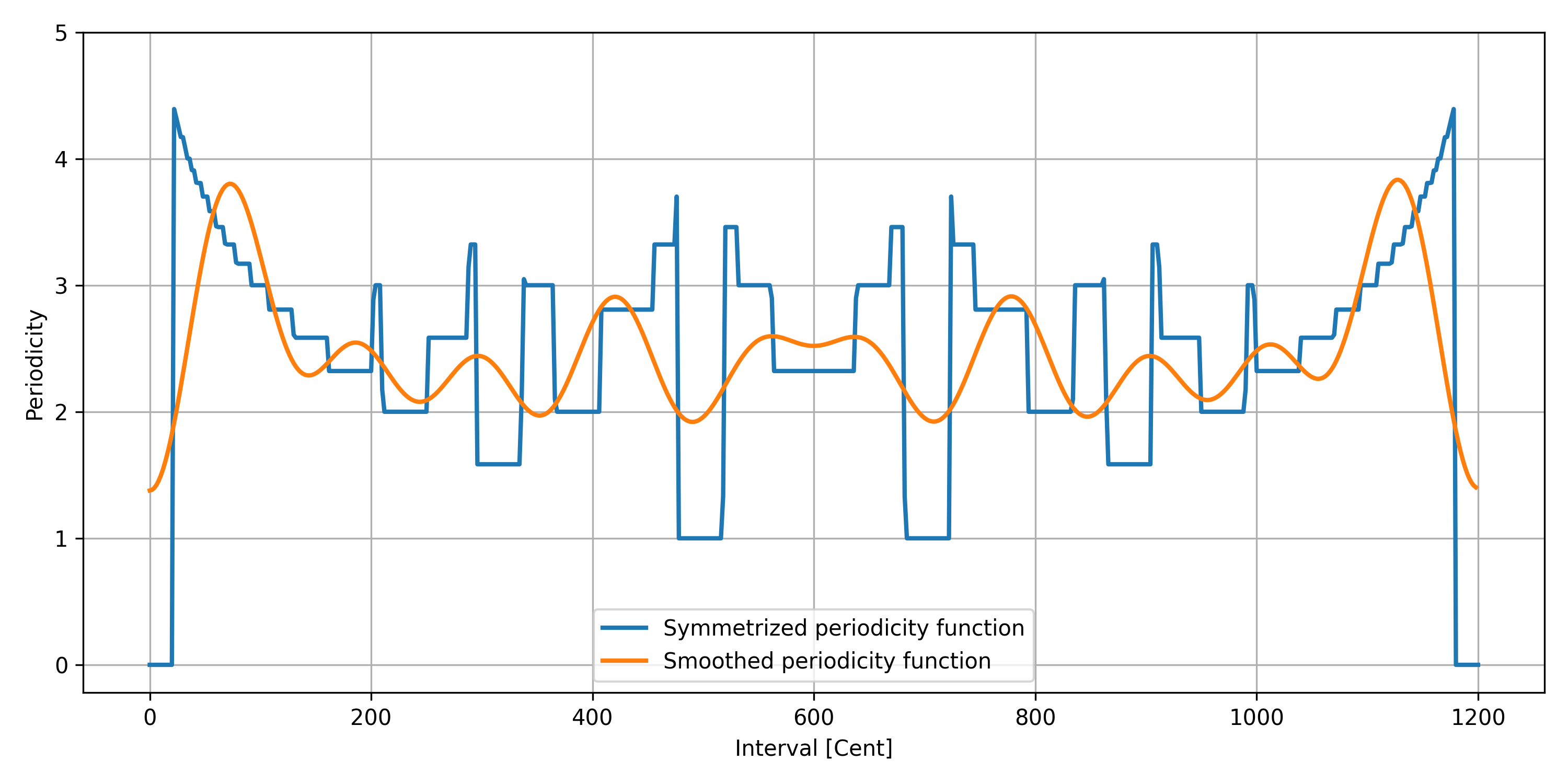}
    \caption{Comparison of the symmetrized and smoothed logarithmic periodicity function.}
    \label{fig:comparison}
\end{figure}
\section{Discussion and Outlook}
\label{section:discussion_and_outlook}
A discussion follows, as well as an outlook on further theoretical considerations and application-related content. 
\subsection{Discussion}
The discussion splits into a general discussion in the context of GDL and one within the musicological framework of the example.
\subsubsection{Conceptual Implications for Geometric Deep Learning}
The main contribution consists in extending spectral constructions from GDL to orbifold domains through a notion of convolution. By working directly with the orbifold geometry, the resulting spectral representation is canonical and coordinate free, and remains compatible with the nontrivial identifications induced by the quotient. This allows spectral convolution and smoothing to be formulated intrinsically on orbifolds, without ad hoc modifications or special treatment of exceptional regions, thereby showing that core mechanisms of GDL extend beyond the manifold setting.

Orbifold geometry encodes identification constraints directly at the level of the domain via quotient constructions, thereby introducing a geometric inductive bias. In contrast to manifold-based settings, where consistency with such identifications is typically approximated through data augmentation, the orbifold formulation restricts the admissible function space a priori to representations that are compatible with the quotient structure. This reduces the effective hypothesis space and ensures intrinsic consistency without additional enforcement mechanisms \citep{Esteves2018}.

\subsubsection{Perceptual Implications of Spectral Smoothing}
The Dirichlet energy
\[
E(f) := \int_{\mathcal{C}^{2}_{12}} \|\nabla_{\mathcal{C}^{2}_{12}} f\|^{2} \, dv_{\mathcal{C}^{2}_{12}}
\]
defines a energy functional for functions on the dyad orbifold. The Laplace eigenfunctions $\psi_k$ can equivalently be characterized as minimizers of this energy under normalization and orthogonality constraints, as in the manifold case; see \citep[p.~53]{bronstein2021geometric}. Spectral smoothing that truncates a spectral expansion to the subspace $V_n$ thus retains those components corresponding to the smallest eigenvalues, and hence to the lowest Dirichlet energy. The resulting representations are therefore optimal with respect to this energy functional within the chosen spectral subspace.

From this perspective, spectral smoothing on the dyad orbifold is not a heuristic post-processing step but the restriction of a function to the subspace spanned by the lowest-energy modes of the orbifold Laplacian. The resulting representations are intrinsic to the orbifold geometry and provide a canonical notion of smoothness and regularity that is induced solely by the underlying energy. Speculatively, such low-energy representations may also be advantageous from a perceptual perspective, as they emphasize stable structures while suppressing high-energy fluctuations that are unlikely to be reliably auditorily discriminated.

\subsection{Outlook}
It was shown that a spectral convolution operator familiar from manifold-based GDL transfers directly to orbifold domains when defined intrinsically via the Laplacian. While in the present work this construction is primarily used as a smoothing mechanism, a direction for future work is to promote it to a trainable building block within end-to-end deep learning architectures on orbifolds. In analogy to spectral CNNs on manifolds, this would involve learning parametric filter functions defined on the Laplace spectrum, stacking multiple such convolutional layers with pointwise nonlinearities and suitable pooling or readout operations, following the general architectural paradigm of GDL \citep{bronstein2017geometric,bronstein2021geometric},  \citep{Cohen2016,Cohen2019}. 
From an optimization perspective, this setting fits into the the framework of learning in orbifolds where loss functions are lifted to the representation space and optimized using stochastic generalized-gradient learning that respect the group-induced identifications via alignment of vector representations\citep{Jain2012}. Taken together, these ideas represent an initial step toward making orbifold domains accessible to CNNs, thereby extending GDL architectures beyond manifolds to spaces with intrinsic quotient geometry.

With regard to specific applications, the present framework has clear relevance for computational music theory, as demonstrated in this work. A further, conceptually distinct direction arises in string phenomenology \citep{He2021}, where physically viable regions of the vast string landscape are explored using machine learning techniques. In this context, several promising corners of the landscape, such as the $\mathbb{Z}_6$-II orbifold compactification of the $E_8 \times E_8$ heterotic string, are described by orbifold geometries \citep{Mutter2019}. Neural networks are already being employed to identify admissible string compactifications and to assist the classification of heterotic orbifold models \citep{Notario2024}. Thus, the search for consistent string compactifications also represents a structurally well-aligned application domain for GDL on orbifolds.

\acks{
This work was supported by a grant from the Vector Foundation (Vector-Stiftung) awarded to B.H. The authors declare no competing financial interests.
}

\end{document}